\documentclass[10pt,twocolumn,letterpaper]{article}

\usepackage{cvpr}
\usepackage[utf8]{inputenc}
\usepackage{graphicx}
\usepackage{color}
\usepackage{times}
\usepackage{amsmath,mathtools}
\usepackage{amsthm}	
\usepackage{amssymb}
\usepackage{enumerate}
\usepackage{subfigure}
\usepackage{algorithm} 
\usepackage{algpseudocode}
\usepackage{booktabs}
\usepackage{tabularx}
\usepackage{authblk}
\usepackage{array}
\usepackage{bm}
\usepackage{url}
\usepackage{bbm}
\usepackage{enumitem}
\usepackage{stackengine}
\usepackage[algo2e,inoutnumbered,linesnumbered,algoruled,vlined]{algorithm2e}
\usepackage{multicol}
\usepackage{xcolor}

\usepackage[breaklinks=true,bookmarks=false]{hyperref}

\usepackage{pifont}
\usepackage[symbol*]{footmisc}

\DefineFNsymbolsTM{otherfnsymbols}{%
  \textbullet \circ
  \textsection   \mathsection
  \textdagger    \dagger
  \textdaggerdbl \ddagger
  \textasteriskcentered *
  \textbardbl    \|%
  \textparagraph \mathparagraph
}%

\setfnsymbol{otherfnsymbols}
\newcommand{\cmark}{\ding{61}}%
\newcommand{\xmark}{*}%
\newcommand{\smark}{\ding{69}}%

\cvprfinalcopy

\usepackage{cleveref}

\def\cvprPaperID{155} 

\ifcvprfinal\fi
\setlist[enumerate]{wide=0pt, leftmargin=15pt, labelwidth=0pt, align=left}

\newtheoremstyle{break}
  {\topsep}{\topsep}%
  {\itshape}{}%
  {\bfseries}{}%
  {\newline}{}%

\theoremstyle{break}

\newcommand{\insertimageC}[5]{ 
\begin{figure}[#5]
\centering
\includegraphics[width=#1\linewidth, clip=true]{figures/#2}
\caption{#3}
\label{#4}
\end{figure}
}

\newcommand{\insertimageStar}[5]{ 
\begin{figure*}[#5]
\centering
\includegraphics[width=#1\linewidth, clip=true]{figures/#2}
\caption{#3}
\label{#4}
\end{figure*}
}

\algnewcommand\algorithmicinput{\textbf{Input:}}
\algnewcommand\INPUT{\item[\algorithmicinput]}
\algnewcommand\algorithmicoutput{\textbf{Output:}}
\algnewcommand\OUTPUT{\item[\algorithmicinput]}

\newcommand{\comment}[1]{}

\let\emptyset\varnothing

\newcommand{\suchthat}{\;\ifnum\currentgrouptype=16 \middle\fi|\;}

\newcommand{\x}{\mathbf{x}}
\newcommand{\X}{\mathbf{X}}

\newcommand{\R}{\mathbb{R}}
\newcommand{\M}{\mathcal{M}}

\newcommand{\Pm}{\mathbf{P}}

\newcommand{\p}{\mathbf{p}}

\newcommand{\f}{\mathbf{f}}
\newcommand{\F}{\mathbf{F}}
\newcommand{\U}{\mathbf{U}}

\newcommand{\Grid}{\mathbf{G}}

\newtheorem{thm}{Theorem}

\newtheorem{dfn}{Definition}

\crefname{section}{§}{§§}
\Crefname{section}{§}{§§}

\crefname{section}{§}{§§}
\Crefname{section}{§}{§§}
\crefname{thm}{Thm.}{Thm}
\crefname{eq}{Eq.}{Eq}
\crefname{figure}{Fig.}{Figure}
\crefname{table}{Tab.}{Table}
\crefname{dfn}{Dfn.}{Dfn}

\begin{document}

\makeatletter
\newcommand{\printfnsymbol}[1]{%
  \textsuperscript{\@fnsymbol{#1}}%
}
\makeatother
\title{3D Point Capsule Networks}

\author[ \cmark\,\,\smark\,\,\thanks{\textit{First two authors contributed equally to this work.}}]{Yongheng Zhao}
\author[ \cmark\,\,\xmark\,\,\printfnsymbol{1}]{Tolga Birdal}
\author[ \cmark\,\,\xmark]{Haowen Deng}
\author[ \cmark]{Federico Tombari\vspace{-5pt}}
\affil[\cmark]{~Technische Universit\"{a}t M\"{u}nchen, Germany\quad $^\text{\smark}$~University of Padova, Italy}
\affil[\xmark]{~Siemens AG, M\"{u}nchen, Germany}

\maketitle

\begin{abstract}
In this paper, we propose 3D point-capsule networks, an auto-encoder designed to process sparse 3D point clouds while preserving spatial arrangements of the input data. 3D capsule networks arise as a direct consequence of our novel unified 3D auto-encoder formulation. Their dynamic routing scheme~\cite{sabour2017dynamic} and the peculiar 2D latent space deployed by our approach bring in improvements for several common point cloud-related tasks, such as object classification, object reconstruction and part segmentation as substantiated by our extensive evaluations. Moreover, it enables new applications such as part interpolation and replacement.
\end{abstract}

\section{Introduction}
\label{sec:intro}
Fueled by recent developments in robotics, autonomous driving and augmented/mixed reality, 3D sensing has become a major research trend in computer vision. Conversely to RGB cameras, the sensors used for 3D capture provide rich geometric structure, rather than high-fidelity appearance information. 
This is proved advantageous for those applications where color and texture are insufficient to accomplish the given task, such as reconstruction/detection of texture-less objects. Unlike the RGB camera case, 3D data come in a variety of forms: range maps, fused RGB-D sequences, meshes and point clouds, volumetric data. Thanks to their capability of representing a sparse 3D structure accurately while being agnostic to the sensing modality, point clouds have been a widespread choice for 3D processing.

The proliferation of deep learning has recently leaped into the 3D domain and architectures for consuming 3D points have been proposed either for volumetric~\cite{Riegler2017CVPR} or sparse~\cite{qi2017pointnet} 3D representations. These architectures overcame many challenges brought in by 3D data, such as order-invariance, complexity due to the added data dimension and local density variations. Unfortunately they often discard spatial arrangements in data, hence falling short of respecting the parts-to-whole relationship, which is critical to explain and describe 3D shapes; maybe even more severe than in the 2D domain due to the increased dimensionality~\cite{bellman2013dynamic}. 
\insertimageC{1}{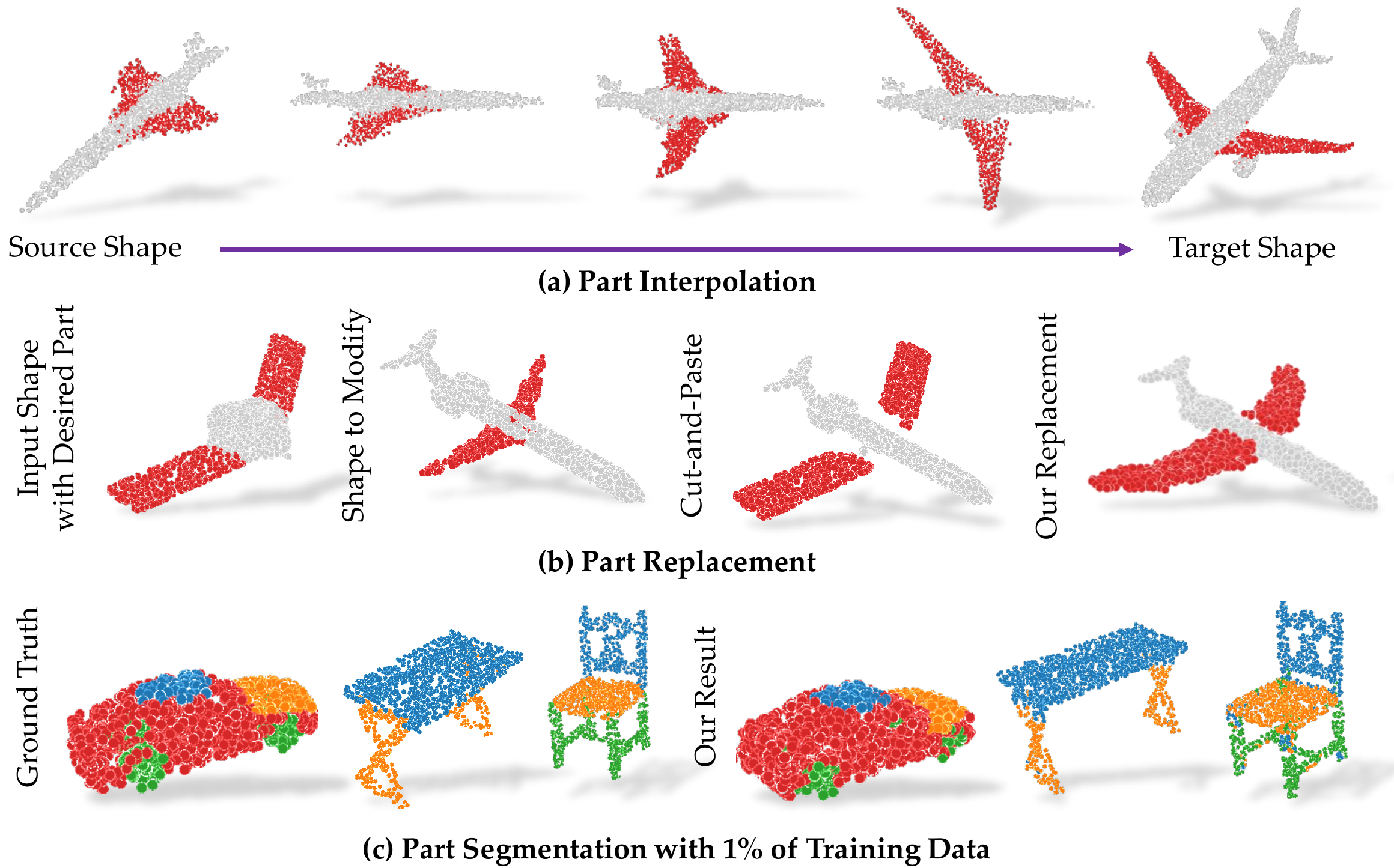}{Our \textit{3D-PointCapsNet} improves numerous 3D tasks while enabling interesting applications such as latent space part interpolation or complete part modification, an application where a simple cut-and-paste results in inconsistent outputs.\vspace{-5mm}}{fig:teaser}{t!}


In this work we first present a unified look to some well known 3D point decoders. Within this view, and based on the renowned 2D capsule networks (CN)~\cite{sabour2017dynamic}, we propose the unsupervised \textit{3D point-capsule networks} (3D-PointCapsNet), an auto-encoder for generic representation learning in unstructured 3D data. Powered by the built-in routing-by-agreement algorithm~\cite{sabour2017dynamic}, our network respects the geometric relationships between the parts, showing better learning ability and generalization properties. We design our 3D-PointCapsNet architecture to take into account the sparsity of point clouds by employing PointNet-like input layers~\cite{qi2017pointnet}. Through an unsupervised dynamic routing, we organize the outcome of multiple max-pooled feature maps into a powerful latent representation. This intermediary latent space is parameterized by \textit{latent capsules} - stacked latent activation vectors specifying the features of the shapes and their likelihood.

Latent capsules obtained from point clouds alleviate the restriction of parameterizing the latent space by a single, low dimensional vector; instead they give explicit control on the basis functions that get composed into 3D shapes.
We further propose a novel 3D point-set decoder operating on these capsules, leading to better reconstructions with increased operational capabilities as shown in~\cref{fig:teaser}. These new abilities stem from the latent capsules instantiating as various shape parameters and concentrating not spatially but semantically across the shape under consideration, even when trained in an unsupervised fashion. We also propose to supply a limited amount of task-specific supervision such that the individual capsules can excel at solving individual sub-problems, e.g. if the task is part-based segmentation, they specialize on different meaningful parts of each shape.

Our extensive quantitative and qualitative evaluation demonstrates the superiority of our architecture. First, we advance the state of the art by a significant margin on multiple frontiers such as 3D local feature extraction, point cloud reconstruction and transfer learning. Next, we show that the distinct attention mechanism of the capsules, driven by dynamic routing, allows a wider range of 3D applications compared to the state of the art auto-encoders: a) part replacement, b) part-by-part animation via interpolation. Note that both of these tasks are non-trivial for standard architectures that rely on 1D latent vectors. Finally, we present improved generalization to unseen data, reaching accuracy levels up to $85\%$ even when using $1\%$ of training data. 
\insertimageStar{0.9725}{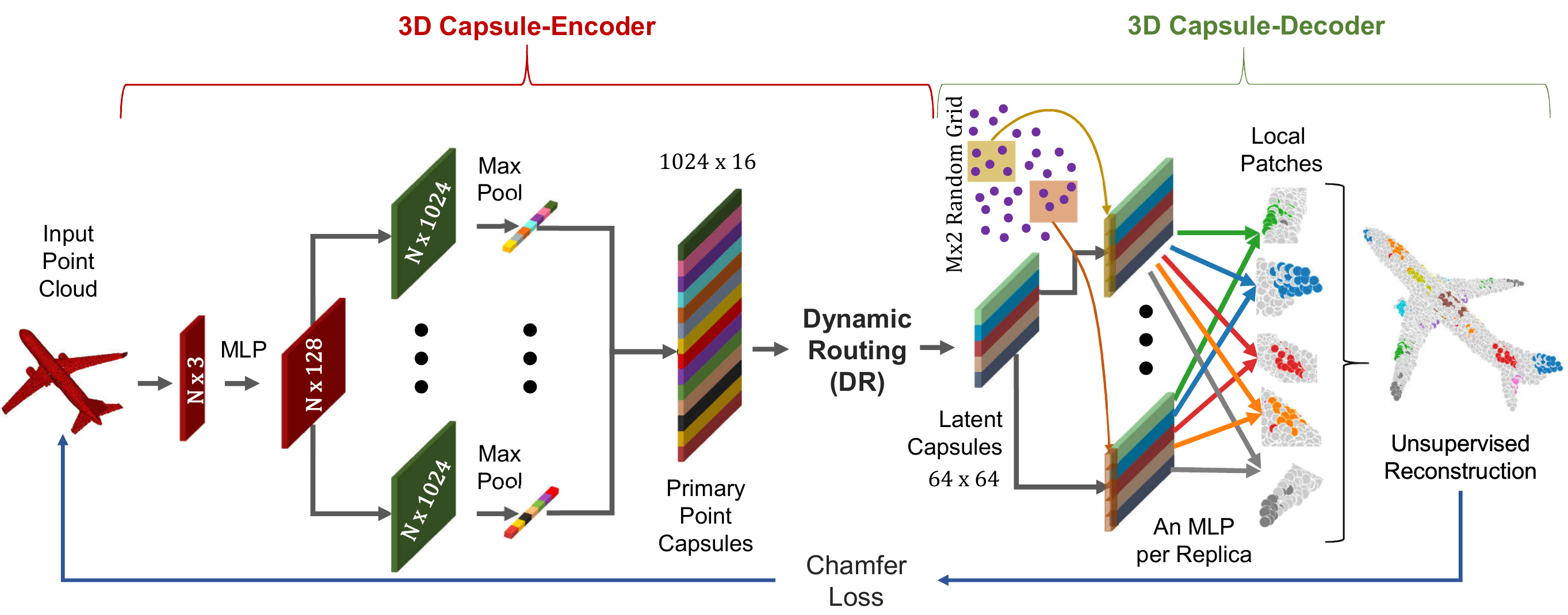}{3D Point Capsule Networks. Our capsule-encoder accepts an $N\times 3$ point cloud as input and uses an MLP to extract $N\times 128$ features from it. These features are then sent into multiple independent convolutional-layers with different weights, each of which is max-pooled to a size of $1024$. The pooled features are then concatenated to form the \textit{primary point capsules} (PPC) ($1024\times 16$). A subsequent dynamic routing clusters the PPC into the final \textit{latent capsules}. Our decoder, responsible for reconstructing point sets given the latent features, endows the latent capsules with random 2D grids and applies MLPs $(64-64-32-16-3)$ to generate multiple point patches. These point patches target different regions of the shape thanks to the DR~\cite{sabour2017dynamic}. Finally, we collect all the patches into a final point cloud and measure the Chamfer distance to the input to guide the network to find the optimal reconstruction. In figure, part-colors encode capsules.\vspace{-3.5mm}}{fig:3dpointcapsnet}{t!}
In a nutshell, our core contributions are: 
\begin{enumerate}[itemsep=0ex]
\item Motivated by a unified perspective of the common point cloud auto-encoders, we propose capsule networks for the realm of 3D data processing as a powerful and effective tool. \item We show that our point-capsule AE can surpass the current art in reconstruction quality, local 3D feature extraction and transfer learning for 3D object recognition.
\item We adapt our latent capsules to different tasks with semi-supervision and show that the latent capsules can master on peculiar parts or properties of the shape. In the end, this paves the way to higher quality predictions and a diverse set of applications like part specific interpolation.
\end{enumerate}\vspace{0.3pt}
Our source code is publicly available under:\\~{\small{\url{https://tinyurl.com/yxq2tmv3}}}\normalfont.

\section{Related Work}
\label{sec:related_work}

\paragraph{Point Clouds in Deep Networks}
Thanks to their generic capability of efficiently explaining 3D data without making assumptions on the modality, point clouds are the preferred containers for many 3D applications~\cite{Zhou_2018_CVPR,naseer2018indoor}. Due to this widespread use, recent works such as PointNet~\cite{qi2017pointnet}, PointNet++~\cite{Qi2017PointNetDH}, SO-Net~\cite{li2018so}, spherical convolutions~\cite{lei2018spherical}, Monte Carlo convolutions~\cite{hermosilla2018monte} and dynamic graph networks~\cite{dgcnn} have all devised point cloud-specific architectures that exploited the sparsity and permutation-invariant properties of 3D point sets. It is also common to process point sets by using local projections reducing the convolution operation down to two dimensions ~\cite{tatarchenko2018tangent,huang2018learning}.

Recently, unsupervised architectures followed up on their supervised counterparts. PU-Net~\cite{yu2018pu} proposed better upsampling schemes to be used in decoding. FoldingNet~\cite{Yang_2018_CVPR} introduced the idea of deforming a 2D grid to decode a 3D surface as a point set. PPF-FoldNet~\cite{Deng_2018_ECCV} improved upon the supervised PPFNet~\cite{deng2018ppfnet} in local feature extraction by benefiting from FoldingNet's decoder~\cite{Yang_2018_CVPR}. AtlasNet~\cite{groueix2018} can be seen as an extension of FoldingNet to multiple grid patches and provided extended capabilities in data representation. 
PointGrow~\cite{sun2018pointgrow} devised an auto-regressive model for both unconditional and conditional point cloud generation leading to effective unsupervised feature learning.
Achlioptas~\etal~\cite{achlioptas2017latent_pc} adapted GANs to 3D point sets, paving the way to enhanced generative learning. 

\vspace{-1.5mm}\paragraph{2D Capsule Networks}
Thanks to their general applicability, capsule networks (CNs) have found tremendous use in 2D deep learning. 
LaLonde and Bagci~\cite{lalonde2018capsules} developed a deconvolutional capsule network, called \textit{SegCaps}, tackling object segmentation.
Durate~\etal~\cite{duarte2018videocapsulenet} extended CNs to action segmentation and classification by introducing~\textit{capsule-pooling}.
Jaiswal~\etal~\cite{jaiswal2018capsulegan}, Saqur~\etal~\cite{saqur2018capsgan} and Upadhyay~\etal~\cite{upadhyay2018generative} proposed Capsule-GANs, \ie capsule network variants of the standard generative adversarial networks (GAN)~\cite{goodfellow2014generative}. These have shown better 2D image generation performance. Lin~\etal~\cite{lin2018learning} showed that capsule representations learn more meaningful 2D manifold embeddings than neurons in a standard CNN do.

There have also been significant improvements upon the initial CN proposal. Hinton~\etal improved the routing by EM algorithm ~\cite{hinton2018matrix}. 
Wang and Liu saw the routing as an optimization minimizing a combination
of clustering-like loss and a KL regularization term~\cite{wang2018optimization}. 
Chen and Crandall~\cite{chen2018generalized} suggested \textit{trainable routing} for better clustering of capsules.
Zhang~\etal~\cite{zhang2018fast} unified the existing routing methods under one umbrella and proposed weighted kernel density estimation based routing methods.
Zhang~\etal~\cite{zhang2018cappronet} chose to use the norm to explain the existence of an entity and proposed to learn a group of capsule subspaces onto which an input feature vector is projected. 
Lenssen~\etal~\cite{lenssen2018group} introduced guaranteed equivariance and invariance properties to capsule networks by the use of group convolutions. 

\vspace{-1.5mm}\paragraph{3D Capsule Networks} Up until now, the use of the capsule idea in the 3D domain has been a rather uncharted territory. Weiler~\etal~\cite{weiler20183d} rigorously formalized the convolutional capsules and presented a convolutional neural network (CNN) equivariant to rigid motions. Jimenez~\etal~\cite{jimenez2018capsule} as well as Mobniy and Nguyen~\cite{mobiny2018fast} extended capsules to deal with volumetric medical data. VideoCapsuleNet~\cite{duarte2018videocapsulenet} also used a volumetric representation to handle temporal frames of the video. Yet, to the best of our knowledge, we are the first to devise a capsule network specifically for 3D point clouds, exploiting their sparse and unstructured nature for representing 3D surfaces.
\section{Method}
\label{sec:method}
 \begin{figure*}[!th]
   \centering
    \includegraphics[width=0.99\linewidth]{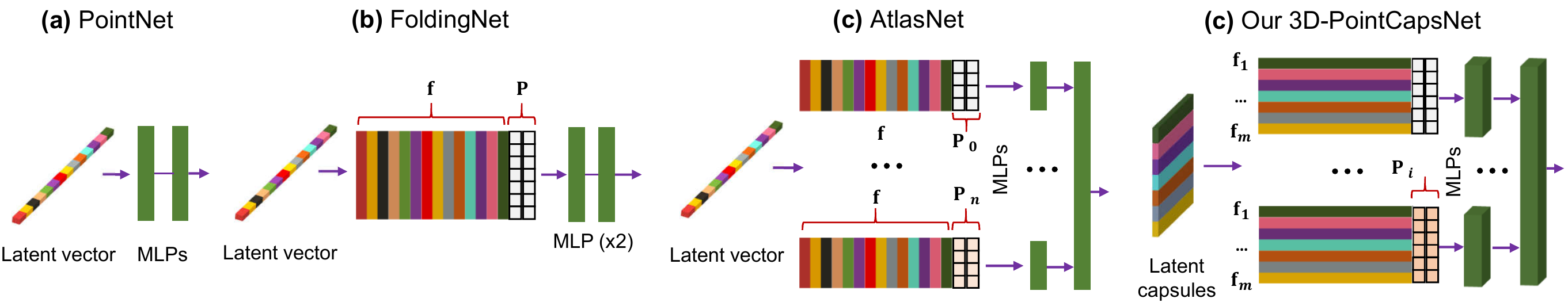}
    \caption{Comparison of four different state-of-the-art 3D point decoders. PointNet uses a single latent vector, and no surface assumption. Thus, $\bm{\theta}_{\text{pointnet}}=\f$. FoldingNet~\cite{Yang_2018_CVPR} learns a 1D latent vector along with a fixed 2D grid $\bm{\theta}_{\text{folding}}=\{ \f,\,\Pm\}$. The advanced AtlasNet~\cite{groueix2018} learns to deform multiple 2D configurations onto local 2-manifolds: $\bm{\theta}_{\text{atlas}}=\{ \f,\,\{\Pm_i\}\}$. Our point-capsule-network is capable of learning multiple latent representations each of which can fold a distinct 2D grid onto a specific local patch, $\bm{\theta}_{\text{ours}}=\{ \{\f_i\},\,\{\Pm_i\}\}$}
    \label{fig:aecompare}
\end{figure*}
\subsection{Formulation}
We first follow the AtlasNet convention~\cite{groueix2018} and present a unified view of some of the common 3D auto-encoders. Then, we explain our \textit{3D-PointCapsNet} within this geometric perspective and justify its superiority compared to its ancestors. We will start by recalling the basic concepts:
\begin{dfn}[Surface and Point Cloud]
\label{dfn:surface}
A 3D surface (\textit{shape}) is a differentiable 2-manifold embedded in the ambient 3D Euclidean space: $\M^2\in \R^3$. We approximate a \textbf{point cloud} as a sampled discrete subset of the surface $\X=\{\x_i \in \M^2 \cap \R^3\}$.
\end{dfn}
\begin{dfn}[Diffeomorphism]
\label{dfn:diff}
A diffeomorphism is a continuous, invertible, structure-preserving map between two differentiable surfaces.
\end{dfn}
\begin{dfn}[Chart and Parametrization]
\label{dfn:chart}
We admit an open set $U\in\R^2$ and a diffeomorphism $C: \M^2 \mapsto U\in \R^2$ mapping an open neighborhood in 3D to its 2D embedding. $C$ is called a \textbf{chart}. Its inverse, $\Psi \equiv C^{-1}: \R^2\mapsto \M^2$ is called a \textbf{parameterization}.
\end{dfn}
\begin{dfn}[Atlas]
\label{dfn:atlas}
A set of charts with images covering the 2-manifold is called an \textbf{atlas}: $\mathcal{A}=\cup_i C_i(\x_i)$.
\end{dfn}
\noindent A 3D auto-encoder learns to generate a 3D surface $\X \in \M^2 \cap \R^{N \times 3}$. By virtue of~\cref{dfn:chart} $\Psi$ deforms a 2D point set to a surface. The goal of the generative models that are of interest here is to learn $\Psi$ to best reconstruct $\hat{\X}\approx\X$:
\begin{dfn}[Problem]
\label{dfn:problem}
Learning to generate the 2-manifolds is defined as finding function(s) $\Psi(U\,|\,\bm{\theta})\,:\, \Psi(U\,|\,\bm{\theta})\approx \X$~\cite{groueix2018}. $\bm{\theta}$ is a lower dimensional parameterization of these functions: $|\bm{\theta}|<|\X|$.
\end{dfn}
\begin{thm}
\label{thm:folding}
Given that $C^{-1}$ exists, $\Psi$, chosen to be a 3-layer MLP, can reconstruct arbitrary 3D surfaces.
\end{thm}
\begin{proof}
The proof is given in~\cite{Yang_2018_CVPR} and follows from the universal approximation theorem (UAT).
\end{proof}
\begin{thm}
\label{thm:atlasnet2}
There exists an integer $K$ s.t. an MLP with $K$ hidden units universally reconstruct $\X$ up to a precision $\epsilon$.
\end{thm}
\begin{proof}
The proof follows trivially from~\cref{thm:folding} and UAT~\cite{groueix2018}.
\end{proof}
Given these definitions, some of the typical 3D point decoders differentiate by making four choices~\cite{qi2017pointnet,groueix2018,Yang_2018_CVPR}:
\begin{enumerate}[itemsep=-1ex]
	\item An open set $U$ or discrete grid $\U\equiv\Pm=\{\p_i \in \R^2\}$.
	\item Distance function $d(\X,\hat{\X})$ between the reconstruction $\hat{\X}$ and the input shape $\X$.
	\item Parameterization function(s) $\Psi$.
	\item Parameters $(\bm{\theta})$ of $\Psi$: $\Psi(U\,|\, \bm{\theta})$.
\end{enumerate}

One of the first works in this field, PointNet~\cite{qi2017pointnet} is extended naturally to an AE by~\cite{achlioptas2017latent_pc} making arguably the simplest choice. We will refer to this variant as \textit{PointNet}. It lacks the grid structure $U=\emptyset$ and functions $\Psi$ only depend upon a single latent feature: $\Psi(U\,|\,\bm{\theta})=\Psi(\bm{\theta})=\text{MLP}(\cdot\,|\,\f\in\R^k)$. FoldingNet uses a two-stage MLP as $\Psi$ to warp a fixed grid $\Pm$ onto $\X$. A transition from FoldingNet to AtlasNet requires having multiple MLP networks operating on multiple 2D sets $\{\Pm_i\}$ constructed randomly on the domain $]0,1[^2$: $\bm{\mathcal{U}}(0,1)$. These explain the better learning capacity of AtlasNet: different MLPs learn to reconstruct distinct local surface patches by learning different charts.

Unfortunately, while numerous charts can be defined in the case of AtlasNet, all of the methods above still rely on a single latent feature vector, replicated and concatenated with $U$ to create the input to the decoders. However, point clouds are found to consist of multiple basis functions~\cite{Sung2018} and having a single representation governing them all is not optimal. We opt to go beyond this restriction and choose to have a set of latent features $\{\f_i\}$ to capture different, meaningful basis functions. 

With the aforementioned observations we can now re-write the well known 3D auto-encoders and introduce a new decoder formulation:\vspace{3mm}\newline
\begin{minipage}[c]{0.35\columnwidth}
\centering 
\underline{PointNet~\cite{qi2017pointnet}}
\begin{flalign}
\U&=\Pm = \emptyset \nonumber\\
\Psi(\bm{\theta}) &:= \text{MLP}(\cdot)\nonumber\\
\bm{\theta} &:= \f\nonumber\\
d(\X,\hat{\X}) &:=d_\text{{EMD}}(\X,\hat{\X})\nonumber
\end{flalign}
\end{minipage} 
\begin{minipage}[c]{0.55\columnwidth}
\centering 
\underline{AtlasNet~\cite{groueix2018}}
\begin{align}
\U=\{\Pm_i\}&:\Pm_i\in\bm{\mathcal{U}}(0,1)\\
\Psi(\bm{\theta}) &:= \{\text{MLP}_i(\cdot)\}\\
\bm{\theta} &:= \{\f, \{\Pm_i\}\}\\
d(\X,\hat{\X}) &:= d_\text{{CH}}(\X,\hat{\X})
\end{align}
\end{minipage}\vspace{3mm}
\begin{minipage}[c]{0.35\columnwidth}
\centering 
\underline{FoldingNet~\cite{Yang_2018_CVPR}}
\begin{flalign}
\U=\Pm &= \Grid^{M\times M}\nonumber\\
\Psi(\bm{\theta}) :=& \text{MLP}(\text{MLP}(\cdot))\nonumber\\
\bm{\theta} &:= \{\f, \Pm\}\nonumber\\
d(\X,\hat{\X}) &:= d_\text{{CH}}(\X,\hat{\X})\nonumber
\end{flalign}
\end{minipage} 
\begin{minipage}[c]{0.61\columnwidth}
\centering 
\underline{Ours}
\begin{align}
\U=\{\Pm_i\}&:\Pm_i\in\bm{\mathcal{U}}(0,1)\\
\Psi(\bm{\theta}) &:= \{\text{MLP}_i(\cdot)\}\\
\bm{\theta} := \{\F&\triangleq \{\f_i\}, \{\Pm_i\}\}\\
d(\X,\hat{\X}) &:= d_\text{{CH}}(\X,\hat{\X})
\end{align}
\end{minipage}\vspace{3mm}
where $d_\text{{EMD}}$ is the Earth Mover~\cite{rubner2000earth} and $d_\text{{CH}}$ is the Chamfer distance. $\Grid^{M\times M}=\{(i \otimes j) : \forall i,j \in [0,\dots,\frac{M-1}{M}]\}$ is a 2D uniform grid. $\f\in \R^{k}$ represents a k-dimensional latent vector. $\bm{\mathcal{U}}(a,b)$ depicts an open set defined by a uniform random distribution in the interval $]a,b[^2$. 

Note that it is possible to easily mix these choices to create variations\footnotemark[4]\footnotetext[4]{ FoldingNet presents evaluations with random grids in their appendix.}. Though, many interesting architectures only optimize for a single latent feature $\f$. To the best of our knowledge, one promising direction is taken by the capsule networks~\cite{hinton2011transforming}, where multitudes of convolutional filters enable the learning of a collection of \textit{capsules} $\{\f_i\}$ thanks to the dynamic routing~\cite{sabour2017dynamic}. Hence, we learn our parameters $\{\bm{\theta}_i\}$ by devising a new point cloud \textit{capsule decoder} that we coin \textit{3D-PointCapsNet}. We illustrate the choices made by four AEs under this unifying umbrella in~\cref{fig:aecompare}.


\begin{table*}[t!]
  \centering
  \caption{Descriptor matching results (recall) on the standard 3DMatch benchmark~\cite{zeng20163dmatch,Deng_2018_ECCV}.}
  \resizebox{\textwidth}{!}{
    \begin{tabular}{lccccccccc}
    \toprule
          & Kitchen  & Home 1  & Home 2  & Hotel 1  & Hotel 2  & Hotel 3  & Study  & MIT Lab  & Average  \\
          \midrule
          \midrule
     3DMatch~\cite{zeng20163dmatch}  & 0.5751 & 0.7372 & 0.7067 & 0.5708 & 0.4423 & 0.6296 & 0.5616 & 0.5455 & 0.5961 \\
     CGF~\cite{Khoury_2017_ICCV}  & 0.4605 & 0.6154 & 0.5625 & 0.4469 & 0.3846 & 0.5926 & 0.4075 & 0.3506 & 0.4776 \\
     PPFNet~\cite{deng2018ppfnet}  & 0.8972 & 0.5577 & 0.5913 & 0.5796 & 0.5769 & 0.6111 & 0.5342 & 0.6364 & 0.6231 \\
     FoldNet~\cite{Yang_2018_CVPR}  & 0.5949 & 0.7179 & 0.6058 & 0.6549 & 0.4231 & 0.6111 & 0.7123 & 0.5844 & 0.6130 \\
     PPF-FoldNet-2K~\cite{Deng_2018_ECCV}  & 0.7352 & 0.7564 & 0.625 & 0.6593 & 0.6058 & 0.8889 & 0.5753 & 0.5974 & 0.6804 \\
     PPF-FoldNet-5K~\cite{Deng_2018_ECCV}  & 0.7866 & 0.7628 & 0.6154 & 0.6814 & 0.7115 & \textbf{0.9444} & 0.6199 & 0.6234 & 0.7182 \\
     \midrule
     Ours-2K & \textbf{0.8518} & \textbf{0.8333} & \textbf{0.7740} & \textbf{0.7699} & \textbf{0.7308} & \textbf{0.9444} & \textbf{0.7397} & \textbf{0.6494} & \textbf{0.7867} \\
     \bottomrule
    \end{tabular}%
    \label{tab:3dmatchbenchmark}%
    }
    \vspace{-2mm}
\end{table*}%

\subsection{3D-PointCapsNet Architecture}
We now describe the architecture of the proposed \textit{3D-PointCapsNet} as a deep 3D point cloud auto-encoder, whose structure is depicted in~\cref{fig:3dpointcapsnet}. 

\vspace{-3mm}\paragraph{Encoder} The Input to our network is an $N \times d$ point cloud, where we fix $N=2048$ and for typical point sets $d=3$. Similarly to PointNet~\cite{qi2017pointnet}, we use a point-wise Multi-Layer Perceptron (MLP) $(3-64-128-1024)$ to extract individual local feature maps. In order to diversify the learning as suggested by capsule networks, we feed these feature maps into multiple independent convolutional layers with different weights, each with a distinct summary of the input shape with diversified attention. We then max-pool their responses to obtain a global latent representation. These descriptors are then concatenated into a set of vectors named \textit{primary point capsules}, $\F$. Size of $\F$ depends upon the size $S_c:=1024$ and the number $K:=16$ of independent kernels at the last layer of MLP. We then use the dynamic routing~\cite{sabour2017dynamic} to embed the primary point capsules into higher level \textit{latent capsules}. Each capsule is independent and can be considered as a \textit{cluster centroid} (codeword) of the primary point capsules. The total size of the latent capsules is fixed to $64\times 64$ (\ie, 64 vectors each sized 64).

\paragraph{Decoder}
Our decoder treats the latent capsules as a feature map and uses MLP$(64-64-32-16-3)$ to reconstruct a patch of points $\hat{\X}_i$, where $|\hat{\X}_i|=64$.
At this point, instead of replicating a single vector as done in~\cite{Yang_2018_CVPR, groueix2018}, we replicate the entire capsule $m$ times and to each replica we append a unique randomly synthesized grid $\Pm_i$ specializing it to a local area. This further stimulates the diversity. We arrive at the final shape $\hat{\X}_i$ by propagating the replicas through a final MLP for each patch and gluing the output patches together. We choose $m =32$ to reconstruct $|\hat{\X}|=32\times 64=2048$ points, the same amount as the input. 
Similar to other AEs, we approximate the loss over 2-manifolds by the discrete Chamfer metric:
\begin{align}
&d_{CH}(\X,\hat{\X}) =\\
&\frac{1}{|\X|}\sum\limits_{\x \in \X} \min_{\hat{\x}\in \hat{\X}} \| \x-\hat{\x} \|_2 + \frac{1}{|\hat{\X}|}\sum\limits_{\hat{\x} \in \hat{\X}} \min_{\x\in\X} \| \x-\hat{\x} \|_2\nonumber
\end{align}
However, this time $\hat{\X}$ follows from the capsules: $\hat{\X} = \cup_i \Psi_i(\Pm_i | \{\f_i\})$. 
\paragraph{Incorporating Optional Supervision}
Motivated by the regularity of capsule distribution over the 2-manifold, we created a capsule-part network that spatially segments the object by associating capsules to parts. The goal here is to assign each capsule to a single part of the object. Hence, we treat this part-segmentation task as a per-capsule classification problem, rather than a per-point one as done in various preceding algorithms~\cite{qi2017pointnet,Qi2017PointNetDH}. This is only possible due to the spatial attention of the capsule networks.

The input of capsule-part network is the latent-capsules obtained from the pre-trained encoder. The output is the part label for each capsule. The ground truth (GT) capsule labeling is obtained from the ShapeNet-Part dataset~\cite{yi2016scalable} in three steps: 1) reconstructing the local part given the capsule and a pre-trained decoder, 2) retrieving the label of the nearest neighbor (NN) GT point for each reconstructed point, 3) computing the most frequent one (mode) among the retrieved labels.

To associate a part to a capsule, we use a shared MLP with a cross entropy loss to classify the latent capsules into parts. This network is trained independently from the 3D-PointCapsNet AE for part supervision. We provide additional architectural details in the supplementary material.

\begin{table*}[t!]
  \centering
  \caption{Descriptor matching results (recall) on the rotated 3DMatch benchmark~\cite{zeng20163dmatch,Deng_2018_ECCV}.}
  \resizebox{\textwidth}{!}{
    \begin{tabular}{lccccccccc}
    \toprule
          & Kitchen  & Home 1  & Home 2  & Hotel 1  & Hotel 2  & Hotel 3  & Study  & MIT Lab  & Average  \\
          \midrule
          \midrule
     3DMatch~\cite{zeng20163dmatch}  & 0.0040 & 0.0128 & 0.0337 & 0.0044 & 0.0000 & 0.0096 & 0.0000 & 0.0260 & 0.0113 \\
     CGF~\cite{Khoury_2017_ICCV}  & 0.4466 & 0.6667 & 0.5288 & 0.4425 & 0.4423 & 0.6296 & 0.4178 & 0.4156 & 0.4987 \\
     PPFNet~\cite{deng2018ppfnet}  & 0.0020 & 0.0000 & 0.0144 & 0.0044 & 0.0000 & 0.0000 & 0.0000 & 0.0000 & 0.0026 \\
     FoldNet~\cite{Yang_2018_CVPR}  & 0.0178 & 0.0321 & 0.0337 & 0.0133 & 0.0096 & 0.0370 & 0.0171 & 0.0260 & 0.0233 \\
     PPF-FoldNet-2K~\cite{Deng_2018_ECCV}  & 0.7352 & 0.7692 & 0.6202 & 0.6637 & 0.6058 & 0.9259 & 0.5616 & 0.6104 & 0.6865 \\
     PPF-FoldNet-5K~\cite{Deng_2018_ECCV}  & 0.7885 & 0.7821 & 0.6442 & 0.6770 & 0.6923 & \textbf{0.9630} & 0.6267 & 0.6753 & 0.7311 \\
     \midrule
     Ours-2K & \textbf{0.8498} & \textbf{0.8525} & \textbf{0.7692} & \textbf{0.8141} & \textbf{0.7596} & 0.9259 & \textbf{0.7602} & \textbf{0.7272} & \textbf{0.8074} \\
     \bottomrule
    \end{tabular}%
  \label{tab:rot3dmatchbenchmark}%
  }
  \vspace{-1mm}
\end{table*}


\section{Experiments}
\label{sec:experiments}

We evaluate our method first quantitatively and then qualitatively on numerous challenging 3D tasks such as local feature extraction, point cloud classification, reconstruction, part segmentation and shape interpolation. We also include a more specific application of \textit{latent space part-interpolation} that is made possible by the use of capsules. For evaluation regarding these tasks, we use multiple benchmark datasets: ShapeNet-Core~\cite{chang2015shapenet}, Shapenet-Part~\cite{yi2016scalable}, ModelNet40~\cite{wu20153d} and 3DMatch benchmark~\cite{zeng20163dmatch}.

\paragraph{Implementation Details}
Prior to training, the input point clouds are aligned to a common reference frame and size normalized. To train our network we use an ADAM optimizer with an initial learning rate of 0.0001 and a batch size of 8. We also employ batch normalization (BN) and RELU activation units at the point of feature extraction to generate primary capsules. Similarly, the multi-stage MLP of the decoder also uses a BN and RELU units except for the last layer, where the activations are scaled by a $tanh(\cdot)$. During dynamic routing operation, we use the squash activation function mentioned in \cite{sabour2017dynamic,hinton2011transforming}.




\begin{table}[t]
  \centering
  \caption{Evaluating reconstruction quality. Oracle refers to a random sampling of the input 3D shape and constitutes an lower bound on what is achievable. The Chamfer Distance is multiplied by $10^3$ for better viewing. CD denotes \textit{Chamfer distance} and PB refers to \textit{Point Baseline}.}
  \vspace{2pt}
  \resizebox{\columnwidth}{!}{
    \begin{tabular}{lccccc}
          & Oracle & PB  & AtlasNet-25 & AtlasNet-125 & Ours \\
          \midrule
    CD    & 0.85  & 1.91  & 1.56  & 1.51  & \textbf{1.46} \\
    \bottomrule
    \end{tabular}%
    \label{tab:reconstruction}%
    }
    \vspace{-3mm}
\end{table}%

\subsection{Quantitative Evaluations}
\paragraph{3D Local Feature Extraction}
We first evaluate 3D Point-Capsule Networks on the challenging task of local feature extraction from point cloud data. 
In this domain, learning methods have already outperformed their handcrafted counterparts by a large margin and hence, we compare only against those, namely 3DMatch~\cite{zeng20163dmatch}, PPFNet~\cite{deng2018ppfnet}, CGF~\cite{Khoury_2017_ICCV}
and PPF-FoldNet~\cite{Deng_2018_ECCV}. PPF-FoldNet is completely unsupervised and yet is still the top performer, thanks to the FoldingNet~\cite{Yang_2018_CVPR} encoder-decoder. It is thus intriguing to see how its performance is affected if one simply replaces its FoldingNet auto-encoder with 3D-PointCapsNet. In an identical setting as~\cite{Deng_2018_ECCV}, we learn to reconstruct the 4 dimensional point pair features~\cite{birdal2015point, birdal2017point} of a local patch, instead of the 3D space of points, and use the latent capsule (codeword) as a 3D descriptor. To restrict the feature vector to a reasonable size of $512$, we limit ourselves only to $16 \times 32$ capsules. 
We then run the matching evaluation on the 3DMatch Benchmark dataset~\cite{zeng20163dmatch} as detailed in~\cite{Deng_2018_ECCV}, and report the recall of correctly founded matches after 21 epochs in~\cref{tab:3dmatchbenchmark}.

We note that our point-capsule networks exhibit an advanced capacity for learning local features, surpassing the state of the art by $10\%$ on the average, even when using $2K$ points unlike the $5K$ of PPF-FoldNet. It is also noteworthy that, except for the Kitchen sequence where PPFNet shows remarkable performance, the recall attained by our network consistently remains above all others. 
We believe that such dramatic improvement is related to the robustness of capsules towards slight deformations in the input data, as well as to our effective decoder.\vspace{-2mm}
\paragraph{Do Our Features Also Perform Well Under Rotation?} 
PPF local encoding of PPF-FoldNet is rotation-invariant. Being based on the same representation, our local feature network should enjoy similar properties. It is of interest to see whether the good performance attained on the standard 3DMatch benchmark transfers to more challenging scenes demanding rotation invariance. To this aim, we repeat the previous assessment on the Rotated-3DMatch benchmark~\cite{Deng_2018_ECCV}, a dataset that introduces arbitrary rotations to the scenes of~\cite{zeng20163dmatch}. Since this dataset contains 6DoF scene transformations, many methods that lack theoretical invariance, e.g. 3DMatch, PPFNet and FoldingNet simply fail. Our unsupervised capsule AE, however, is once again the top performer, surpassing the state of the art by $\sim 12\%$ on $2K$-point case as shown in Tab.~\ref{tab:rot3dmatchbenchmark}. This significant gain justifies that our encoder manages to operate also on the space of 4D PPFs, holding on the theoretical invariances.\vspace{-2mm}
\insertimageStar{.95}{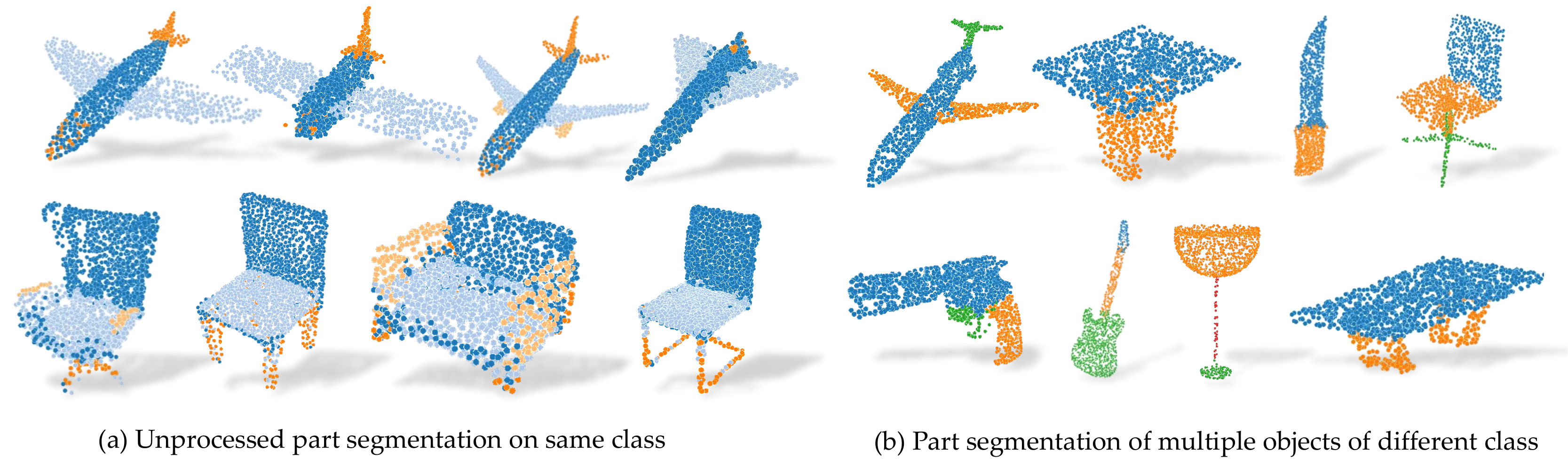}{Part segmentation by capsule association. Having pre-trained the auto-encoder, we append a final part-supervision layer and use a limited amount of data to specialize the capsules on object parts. (\textbf{a}) across the shapes of the same class capsules capture semantic regions. (\textbf{b}) inter-class part segmentation. Colors indicate different capsule groups and (\textbf{b}) uses only a simple median filter to smooth the results.\vspace{-2mm}}{fig:finetuning}{t!}
\paragraph{3D Reconstruction}
In a further experiment, we evaluate the quality of our architecture in point generation. 
We assess the reconstruction performance by the standard Chamfer metric and base our comparisons on the state of the art auto-encoder AtlasNet and its baselines (point-MLP)~\cite{groueix2018}. We rely on the ShapeNet Core v2 dataset~\cite{chang2015shapenet}, using the same training and test splits as well as the same evaluation metric as those in AtlasNet's~\cite{groueix2018}. We show in~\cref{tab:reconstruction} the Chamfer distances averaged over all categories and for $N>2K$ points. 
It is observed that our capsule AE results in lower reconstruction error even when a large number of patches (125) is used in favor of AtlasNet. This justifies that the proposed network has a better summarization capability and can result in higher fidelity reconstructions. \vspace{-1mm}
\begin{table}[t]
  \centering
  \caption{Accuracy of classification by transfer learning on the ModelNet40 dataset. Networks are trained out ShapeNet55, except \textit{Ours-Parts} that is trained on smaller ShapeNet-Parts dataset.}
  \vspace{2pt}
  \resizebox{\columnwidth}{!}
  {
    \begin{tabular}{lcccc}
    & {Latent-GAN\cite{achlioptas2017latent_pc}} & FoldingNet\cite{Yang_2018_CVPR} & Ours-Parts & {Ours} \\
    \toprule
     \text{Acc.} & 85.7& 88.4 & 88.9 & \textbf{89.3} \\
    \bottomrule
    \end{tabular}%
  \label{tab:transfer_cls}%
 }
 \vspace{-1.5mm}
\end{table}%

\paragraph{Transfer Learning for 3D Object Classification}
In this section, we demonstrate the efficiency of learned representation by evaluating the classification accuracy obtained by performing transfer learning. Identical to \cite{wu2016learning,achlioptas2017latent_pc,Yang_2018_CVPR}, we train a linear SVM classifier so as to regress the shape class given the latent features. To do that, we reshape our latent capsules into a one dimensional feature and train the classifier on Modelnet40~\cite{wu20153d}. We use the same train/test split sets as \cite{Yang_2018_CVPR} and obtain the latent capsules by training 3D-PointCapsNet on a different dataset, the ShapeNet-Parts~\cite{yi2016scalable}. The training data has 14,000 models subdivided into 16 classes.
The evaluation result is shown in~\cref{tab:transfer_cls}, where our AE, trained on a smaller dataset compared to the ShapeNet55 of~\cite{achlioptas2017latent_pc,Yang_2018_CVPR} is capable of performing at least on par or better. This shows that learned latent capsules can handle smaller datasets and generalize better to new tasks. We also evaluated our classification performance when the training data is scarce and obtained similar result as the FoldingNet, $\sim85\%$ on $\sim20\%$ of training data.
\begin{table}[t!]
  \centering
  \caption{Part segmentation on ShapeNet-Part by learning only on the $x\%$ of the  training data.}
  \vspace{1pt}
  \resizebox{\columnwidth}{!}{
    \begin{tabular}{lcccc}
    {Metric} & SONet-$1\%$ & Ours-$1\%$ & SONet-$5\%$& Ours-$5\%$ \\
    \toprule
    Accuracy & 0.78 & \textbf{0.85} & 0.84 & \textbf{0.86}\\
    IoU & 0.64 & \textbf{0.67} & 0.69 & \textbf{0.70}\\
    \bottomrule
    \end{tabular}%
  \label{tab:socompare}%
  }
  \vspace{-3mm}
\end{table}%

\subsection{Qualitative results}
\paragraph{3D Object Part Segmentation with Limited Data}
We now demonstrate the regional attention of our latent capsule and their capacity to learn with limited data. To this end, we trained 3D-PointCapsNet on the ShapeNet-Part dataset~\cite{chang2015shapenet} for part segmentation as explained in~\cref{sec:method}, with a supervision of only $1-5\%$ part labeled training data. We tested our network on all of the available test data. To specialize the capsules to distinct parts, we select as many capsules as the part labels and let the per-capsule classification coincide to part predictions. Predicted capsule labels are propagated to the related points. For the sake of space, we compared our results only with the state of the art on this dataset, the SO-Net~\cite{lin2018learning}. We use identical evaluation metrics as SO-Net~\cite{lin2018learning}: \textit{Accuracy} and \textit{IoU} (Intersection over Union), and report our findings in~\cref{tab:socompare}. Note that, when trained on $1\%$ of input data, we perform $7\%$ better than SO-Net. When the amount of training data is increased to $5\%$, the gap closes but we still surpass SO-Net by $2\%$, albeit training a smaller network to classify latent-capsules rather than 3D points.\vspace{-1mm}
\insertimageC{1}{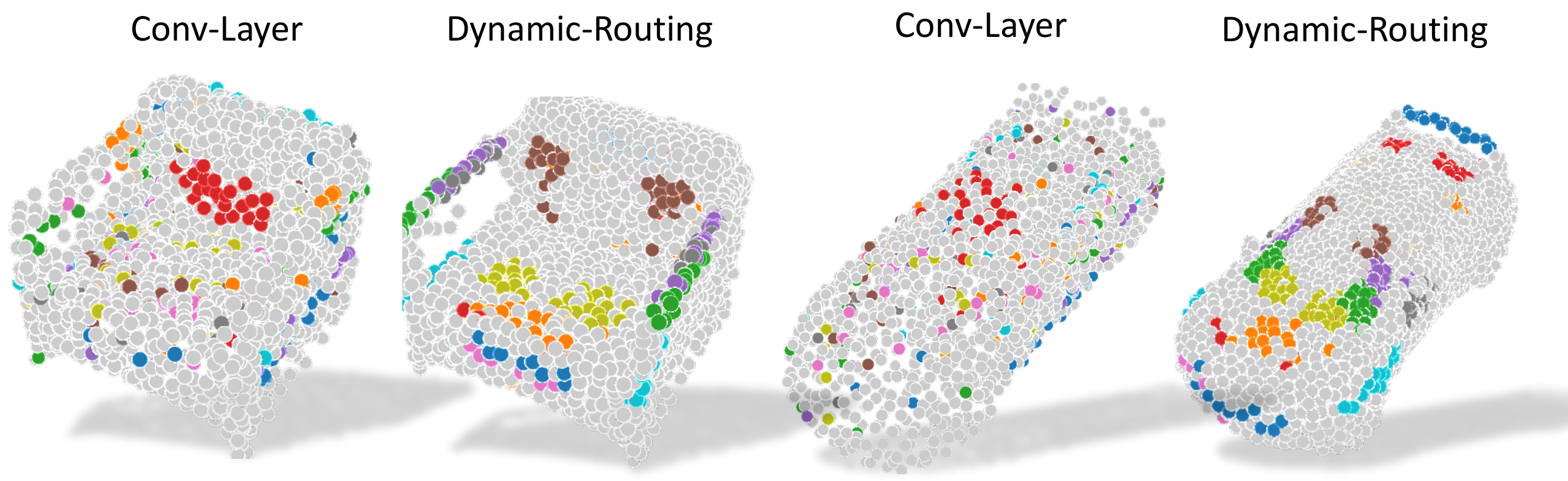}{Distribution of 10 randomly selected capsules on the reconstructed shape after unsupervised autoencoder training \textbf{a}) with dynamic routing, \textbf{b}) with a simple convolutional layer.\vspace{-3mm}}{fig:unsupcaps}{t!}
\insertimageStar{.95}{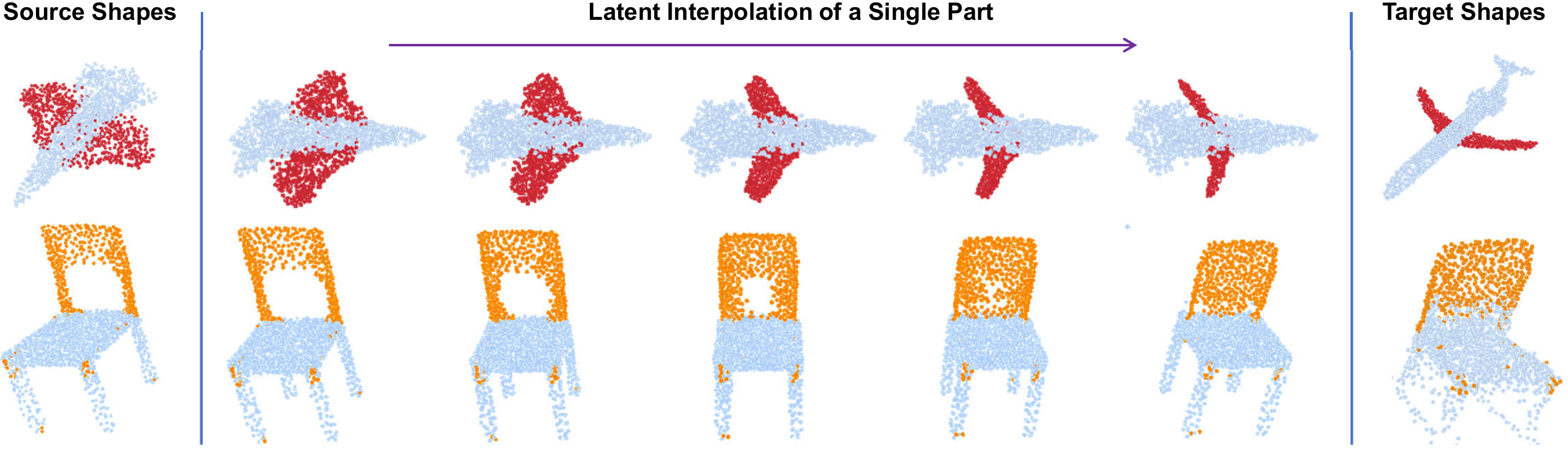}{Part interpolation on the Shapenet-Part~\cite{yi2016scalable} dataset. $\text(\textbf{left})$ The source point cloud. $\text(\textbf{right})$ Target shape. $\text(\textbf{middle})$ Part interpolation. Fixed part is marked in light blue and the interpolated part is highlighted. Capsules are capable of performing part interpolation purely via latent space arithmetic. \vspace{-5mm}}{fig:partinterp}{t!}
\paragraph{Does unsupervised training lead to specialized capsules?}
It is of interest to see whether the original argument of the capsule networks~\cite{sabour2017dynamic,hinton2011transforming} claiming to better capture the intrinsic geometric properties of the object still holds in the case of our unsupervised 3D-AE. To this aim, we first show in~\cref{fig:unsupcaps} that even with lack of supervision the capsules specialize on local parts of the model. While these parts may sometimes not correspond to the human annotated part segmentation of the model, we still expect them to concentrate on semantically similar regions of the 2-manifold.~\cref{fig:unsupcaps} visualizes the distribution of 10 capsules by coloring them individually and validates our argument.

To validate our second hypothesis, stating that such clustering arises thanks to the dynamic routing, we replace the DR part of the AE with standard PointNet-like layers projecting the $1024\times 64$ PPC to 64$^2$ capsules and repeat the experiment.~\cref{fig:unsupcaps} depicts the spread of the latent vectors over the point set when such layer is employed as opposed to DR. Note that using this simple layer instead of DR both harms the reconstruction quality and yields an undesired spread of the capsules across the shape. We leave it as a future work to study the DR energy theoretically and provide more details on this experiment in the supplement.




\vspace{-3mm}
\paragraph{Semi-supervision guides the capsules to meaningful parts.}
We now consider the effect of training in steering the capsules towards the optimal solution in the task of supervised part segmentation. First, we show in~\cref{fig:finetuning} the results obtained by the proposed part segmentation: (\textbf{a}) shows part segmentation across multiple shapes of the same class. These results are also unfiltered and the raw outcome of our network. (\textbf{b}) depicts part segmentation across a set of object classes from Shapenet-Part. It also shows that some minor confusions present in (\textbf{a}) can be corrected with a simple median filter. This is contrary and computationally preferable to costly CRFs smoothing the results~\cite{wang2017cnn}. 

Next, we observe that, as training iterations progress, the randomly initialized capsules specialize to parts, achieving a good part segmentation at the point of convergence. We visualize this phenomenon in~\cref{fig:evolution}, where the capsules that have captured the wings of the airplane are monitored throughout the optimization procedure. Even though the initial random distribution is spatially spread out, the resulting configuration is still part specific. This is a natural consequence of our capsule-wise part semi-supervision. 

\insertimageC{1}{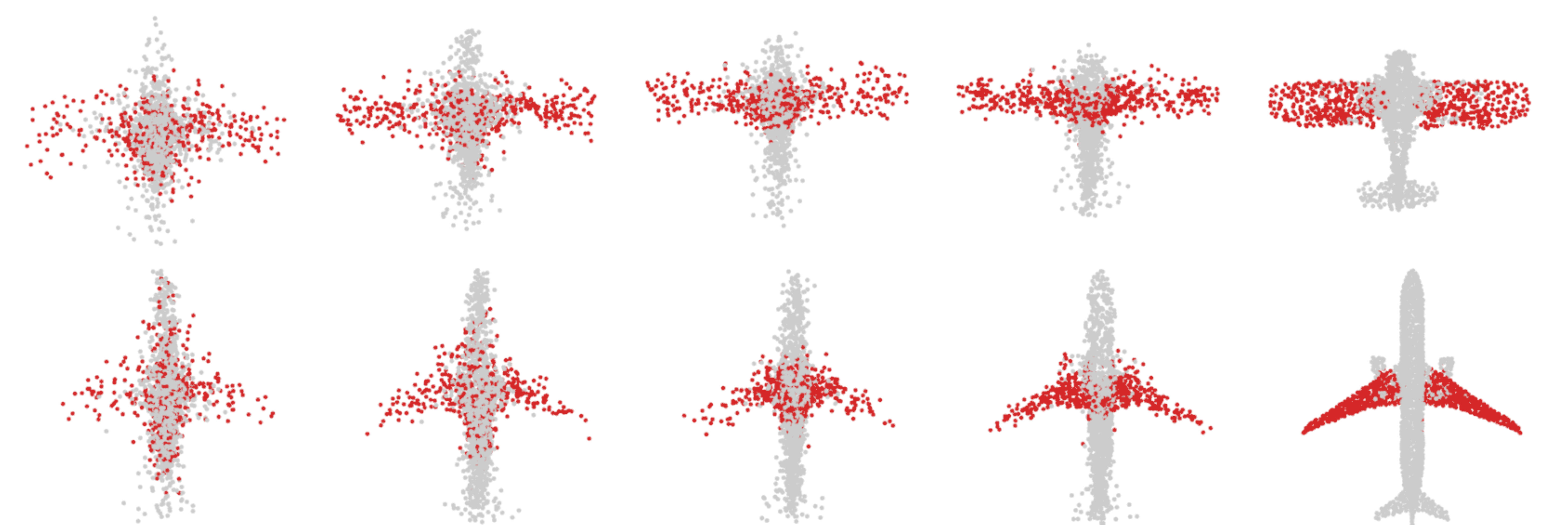}{Visualizing the iterations of unsupervised AE training on the airplane object. For clear visualization, we fetch the colors belonging to the $\sim$20 capsules of the wing-part from our part predictions trained with part supervision.\vspace{-3mm}}{fig:evolution}{t!}
\vspace{-2mm}
\paragraph{Part Interpolation / Replacement}
Finally, we explore the rather uncommon but particularly interesting application of interpolating, exchanging or switching object parts via latent-space manipulation.
Thanks to the fact that 3D-PointCapsNet discovers multiple latent vectors specific to object attributes/shape parts, our network is capable of performing per-part processing in the latent space. To do that, we first spot a set of latent capsule pairs belonging to the same parts of two 3D point shapes and intersect them. Because these capsules explain the same part in multiple shapes, we assume that they are specific to the part under consideration and nothing else.
We then interpolate linearly in the latent space between the selected capsules. As shown in ~\cref{fig:partinterp} the reconstruction of intermediate shapes vary only at a single part, the one being interpolated. When the interpolator reaches the target shape it \textit{replaces} the source part with the target one, enabling \textit{part-replacement}. \cref{fig:partreplace} further shows this in action. Given two shapes and latent capsules of the related parts, we perform a part exchange by simply switching some of the latent capsules and reconstructing. Conducting a part exchange directly on the input space by a cut-and-place would yield inconsistent shapes as the replaced parts would have no global coherence.

\insertimageC{1}{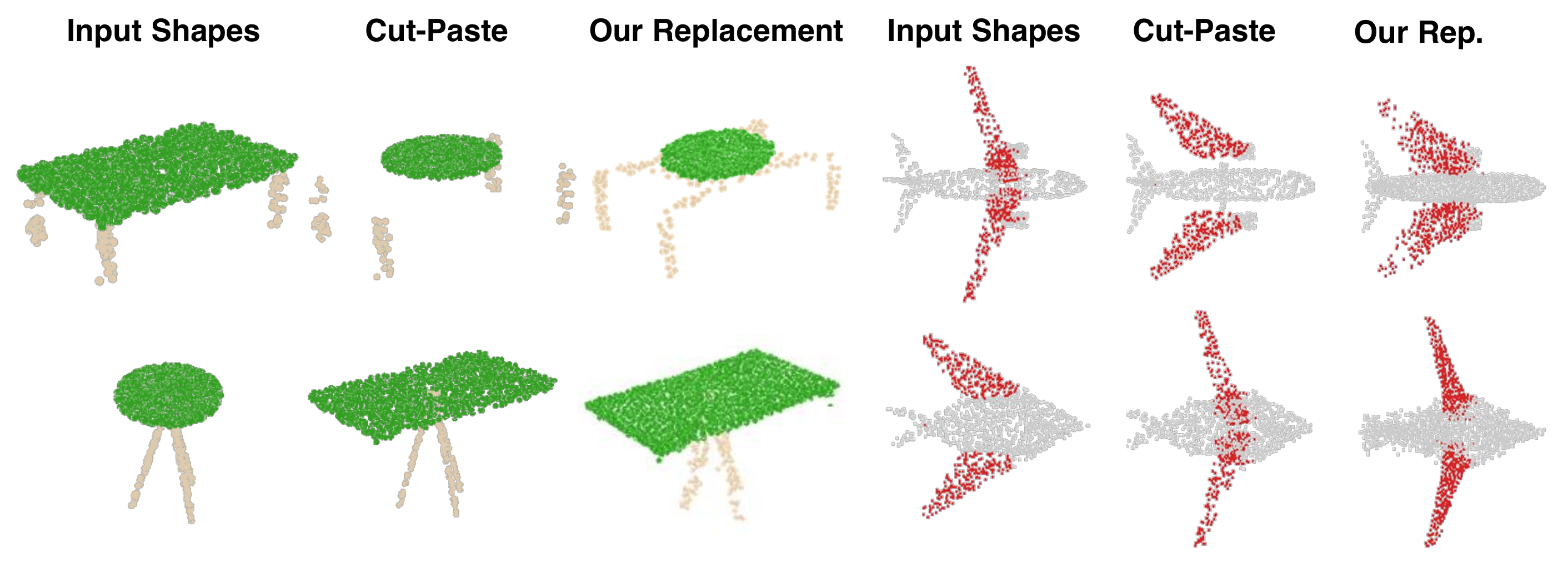}{Part replacement. Performing replacement in the latent space rather than Euclidean space of 3D points yields geometrically consistent outcome.\vspace{-3mm}}{fig:partreplace}{t!}




%
%

\vspace{-2mm}\section{Conclusion}\vspace{-1mm}
\label{sec:conclusion}
We have presented 3D Point-Capsule Network, a flexible and effective tool for 3D shape processing and understanding. We first presented a broad look to the common point cloud AEs. With the observation that a one dimensional latent embedding, the choice of the most preceding auto-encoders, is potentially sub-optimal, we opted to summarize the point clouds as a union of disjoint latent basis functions. We have shown that such choice can be implemented by learning the embedded \textit{latent capsules} via dynamic routing. 
Our algorithm proved successful on an extensive evaluation on many 3D shape processing tasks such as 3D reconstruction, local feature extraction and part segmentation. Having a latent capsule set rather than a single vector also enabled us to address new applications such as part interpolation and replacement. In the future, we plan to deploy our network for pose estimation and object detection from 3D data, currently two of the key challenges in 3D computer vision.

\noindent \textbf{Acknowledgements }
We would like to thank Yida Wang, Shuncheng Wu and David Joseph Tan for fruitful discussions. This work is partially supported by China Scholarship Council (CSC) and IAS-Lab in the Department of Information Engineering of the University of Padova, Italy.    
{\small

}
\setcounter{section}{0}
\renewcommand\thesection{\Alph{section}}
\newcommand{\suppsection}{\subsection}
\section*{Appendix}\nonumber
\section{Semi-supervised Classification}
We begin by showing semi-supervised classification results in~\cref{tab:semi}. Note that our network can generate predictions that are on par with or better than FoldingNet~\cite{Yang_2018_CVPR}. \vspace{-5pt}
\begin{table}[htbp]
  \centering
  \caption{Part segmentation on ShapeNet-Part by learning on limited training data. The table shows the accuracies obtained by FoldingNet~\cite{Yang_2018_CVPR} and our approach for different amount of training data. }
  \resizebox{0.9\columnwidth}{!}{
    \begin{tabular}{lccccc}
     & $1\%$ & $2\%$ & $5\%$& $20\%$& $100\%$ \\
    \toprule
    FoldingNet & 56.15 & 67.05 & 75.97 &84.06 &88.41 \\
    Ours & 59.24 & 67.67 & 76.49 & 84.48 & 88.91\\
    \bottomrule
    \end{tabular}%
  \label{tab:semi}%
  }
\end{table}%

\vspace{-5pt}
\section{Part Segmentation}
We first give a small summary of the part association network for optional supervision. The input to this one-layer architecture is the latent capsules combined with one-hot vector of the object category. The output is the part prediction of each capsule. We use the cross entropy loss as our loss function and Adam as the optimizer with the learning rate of 0.01. The network structure is shown in Fig.~\ref{fig:supp-partseg-pipline}.

Then we utilize the pre-trained decoder to reconstruct the object with the labeled capsules. Fig.~\ref{fig:supp-partseg} depicts further visualizations for different objects from the ShapeNet-Part dataset~\cite{yi2016scalable}. Our results are also qualitatively comparable to ground truth.
\insertimageC{1}{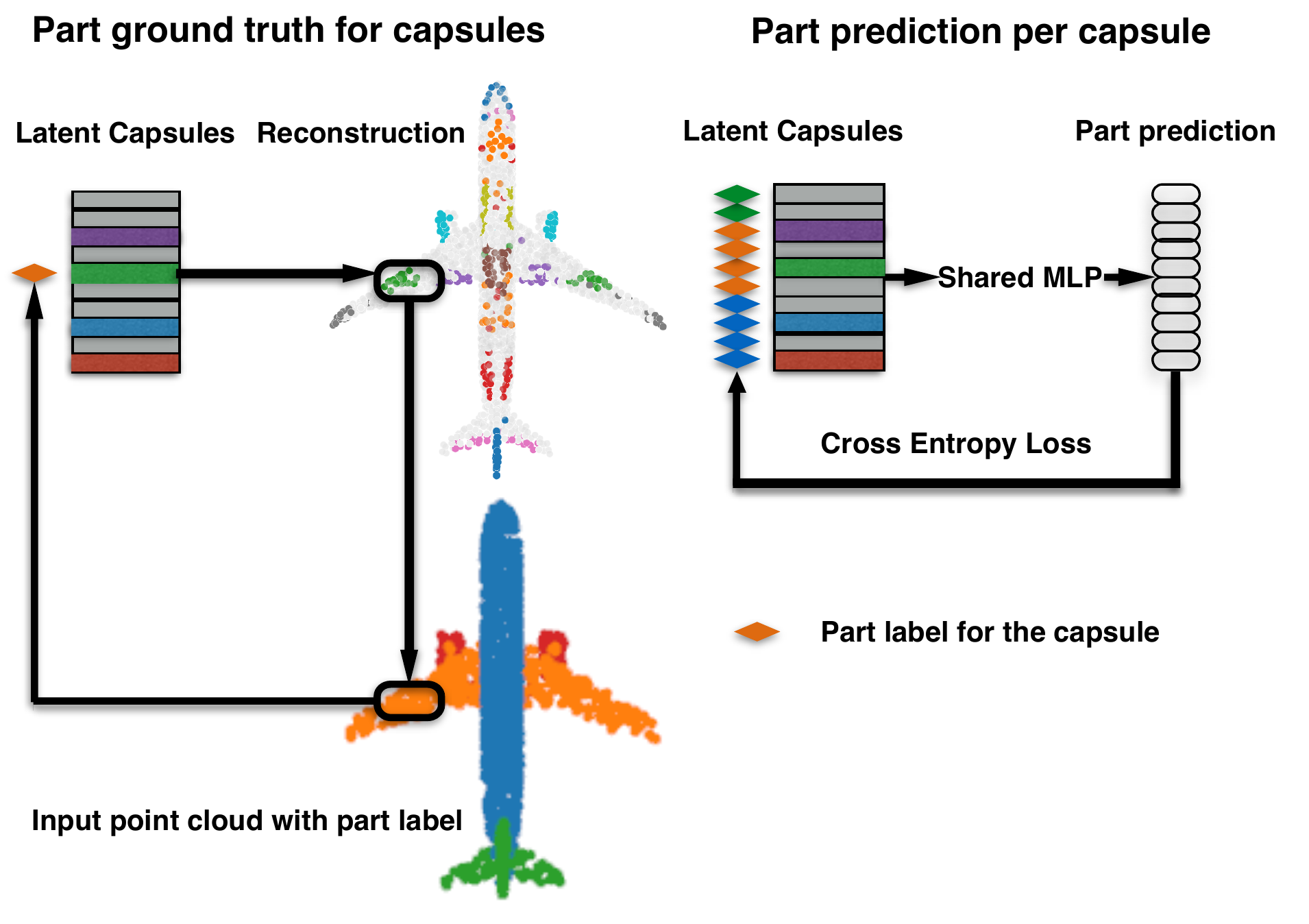}{Supervising the 3d point capsule networks for part prediction. Instead of performing a point-wise part labeling, we use a capsule-wise association requiring less data annotation efforts.}{fig:supp-partseg-pipline}{hbtp}

\section{Part Interpolation}
We first show an overview of how we perform part interpolation. While this part has been thoroughly explained in the paper, we have omitted this architecture illustration due to space considerations. We now provide this in~\cref{fig:pipeline-interp}.

Next we show, the part interpolation results on different objects. In this qualitative evaluation, we are given two shapes and the goal is to interpolate the source part towards the target. To do that we find the matching capsules that represent the part of interest in both shapes. We then linearly interpolate from the capsule(s) of the source to the one(s) of the target. This generates visually pleasing intermediate shapes, which our network has never seen before. Here we see that the learned embedding resemble a Euclidean space where linear latent space arithmetic is possible. It is also visible that such interpolation scheme can handle topological changes such as merging or branching legs. In the end of interpolation a new shape is generated in which the part is replaced completely with the target's. That brings us to our second and interesting application, part replacement.

\section{Part Replacement}
We now supplement our paper by presenting additional qualitative results on the task of part replacement.~\cref{fig:supp-partrep} shows numerous object pairs where a part-of-interest is selected in both and exchanged by the help of latent space capsule arithmetic. Analogous to the ones in the paper we also show a cut-and-paste operation that is a mere exchange of the parts in 3D space, obviously resulting in undesired disconnected shapes. Thanks to our decoder's capability in generating high fidelity shapes, our capsule-replacement respects the overall coherence of the resulting point cloud.

\section{Ablation Study}
In order to show the prosperity of the dynamic routing, we compare the reconstruction result by replacing the DR with PointNet-like set of convolutional layers. In this ablation study, the primary point capsules ($1024\times16$) are considered as 1024 point-features and each point has the feature dimension of 16. We utilize a shared MLP to increase the feature dimension from 16 to 64. After conducting max pooling, we can obtain a vector of length 64. With multiple MLPs and max-pooling, we are able to generate 64 vectors which have the same dimensions as the latent capsules produced by dynamic routing. The structure of this comparison module is shown in Fig.~\ref{fig:supp-convcompare}. To carry out our fair evaluation, we re-train the whole AE with this module. The result of the reconstruction is shown in Fig. 5 of the main paper.
\insertimageC{1}{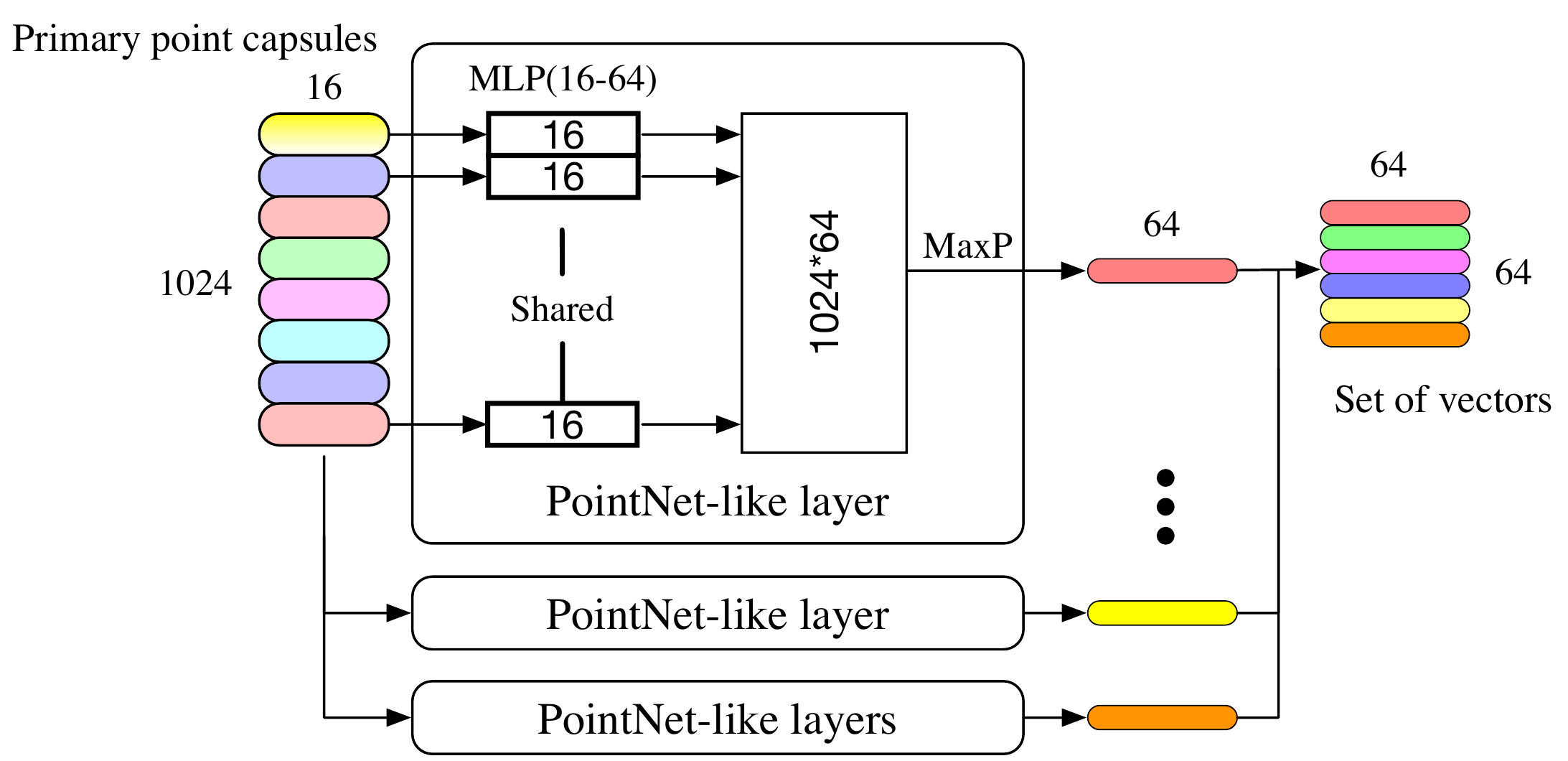}{The structure of the comparison module that operates on the primary point capsules and generates a set of vectors having the same dimensionality as the latent capsule output of DR.}{fig:supp-convcompare}{hbtp}

\section{A Discussion on the Local Spatial Attention}
Our network consists of multiple MLPs acting on a single capsule. It encodes the part information inside that capsule rather than the MLPs themselves.
For that reason, the local attention stems from both the organization of primary point capsules (in our case obtained by dynamic routing) and potentially the decoder (see Fig. 5 of the main paper). Thus, we are able to control and represent the shape instantiation in the latent space as shown in \textit{part interpolation/replacement} evaluations. 
Contrarily, AtlasNet reconstructs different local patches with different MLPs from the same latent vector. This embeds the part knowledge into the MLPs, making it challenging to control the shape properties. 

\insertimageStar{1}{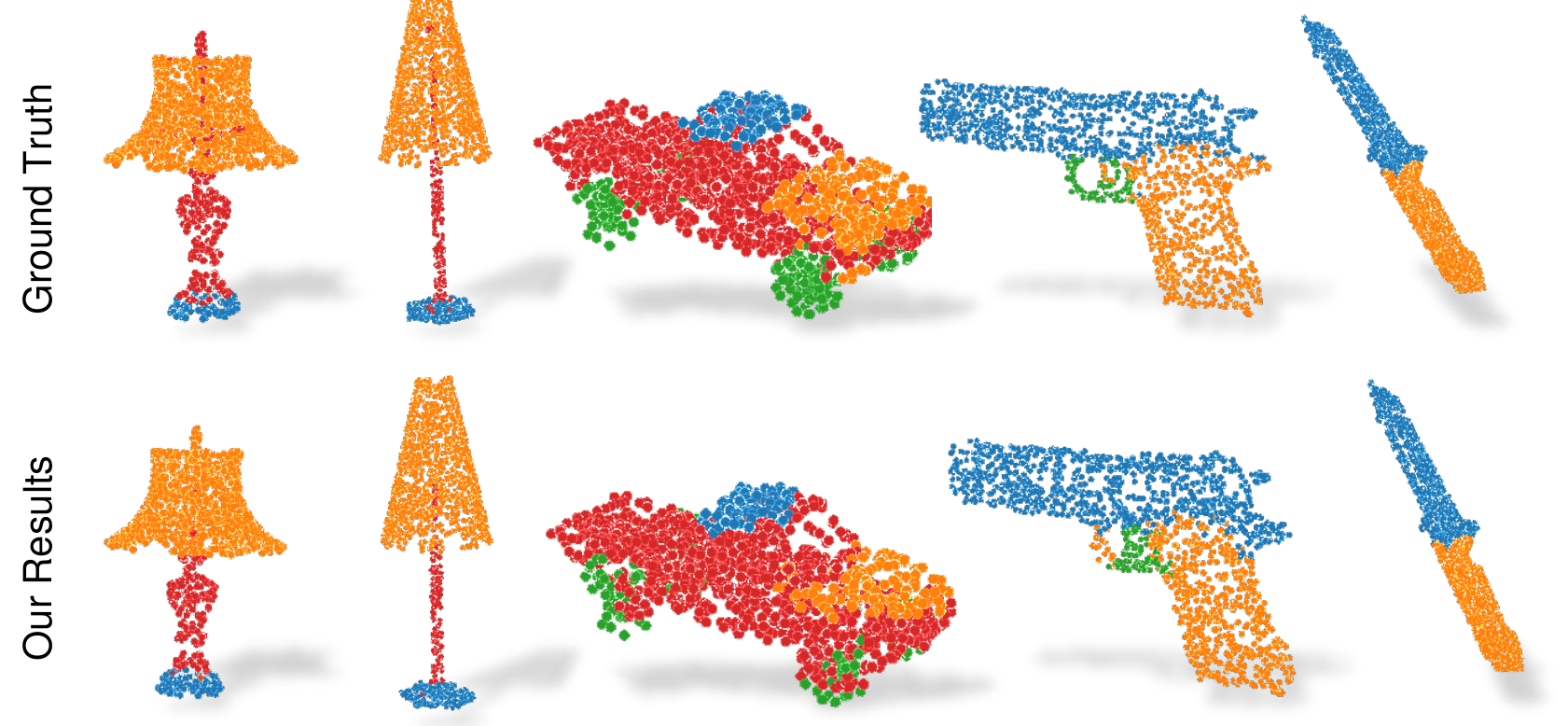}{Part segmentation on limited amount of training data.}{fig:supp-partseg}{b!}
\insertimageStar{1}{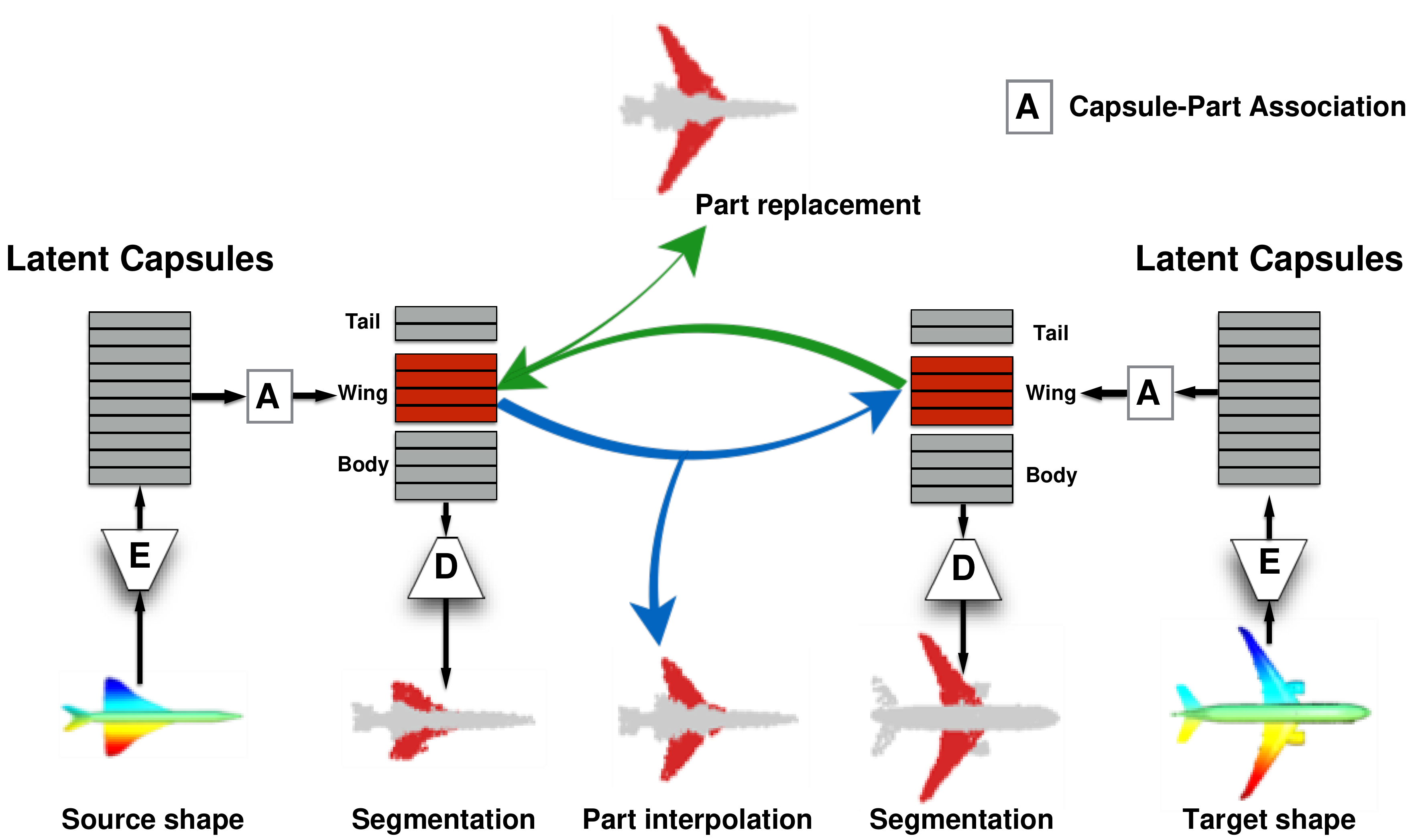}{Our interpolation / replacement pipeline.}{fig:pipeline-interp}{hbtp}
\insertimageStar{1}{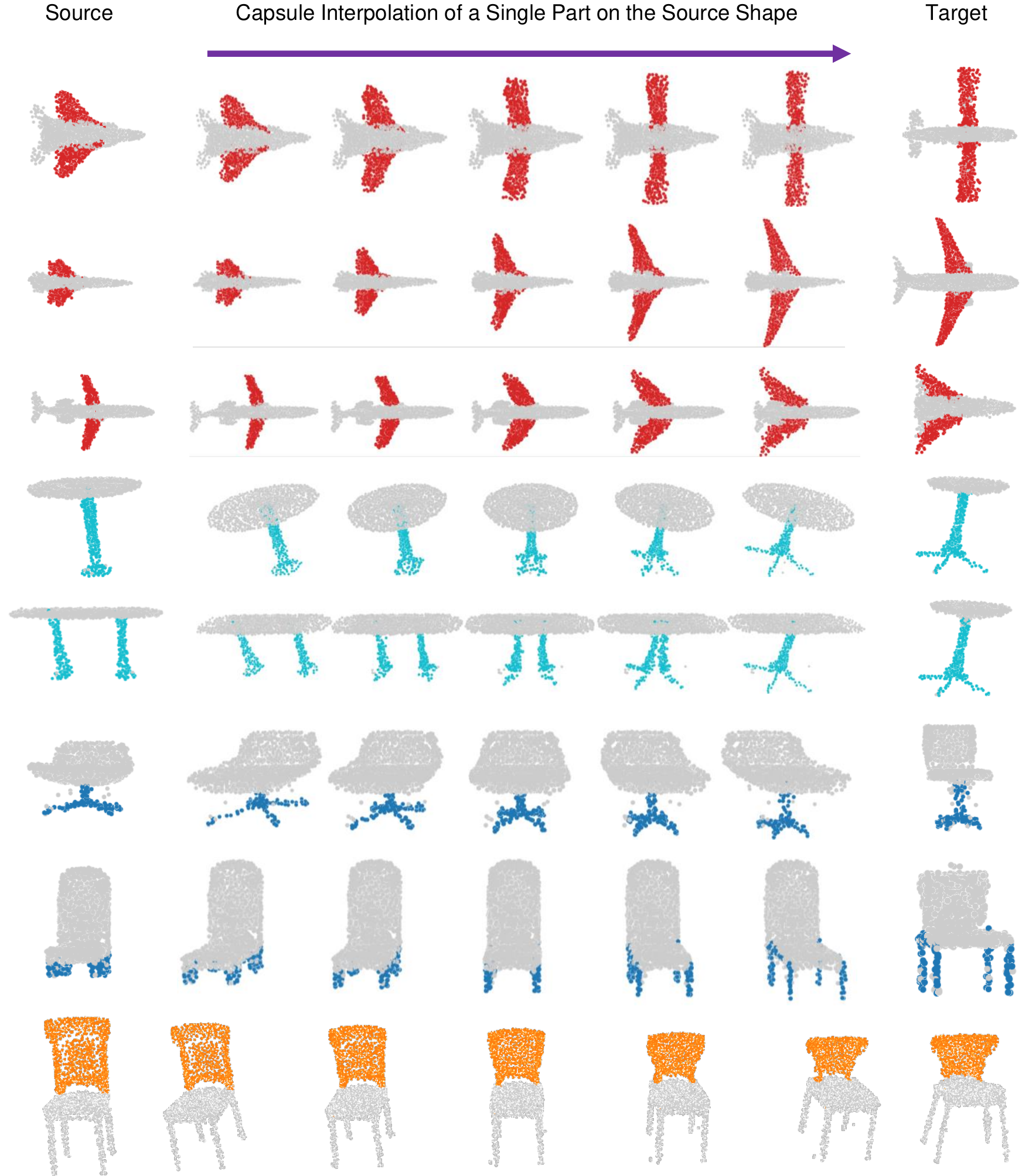}{Visualization of part interpolation from source shape part to target. By simple linear interpolation on the correspondent capsule(s), smooth intermediate topologies could be generated. }{fig:supp-partinterp}{hbtp}

\insertimageStar{1}{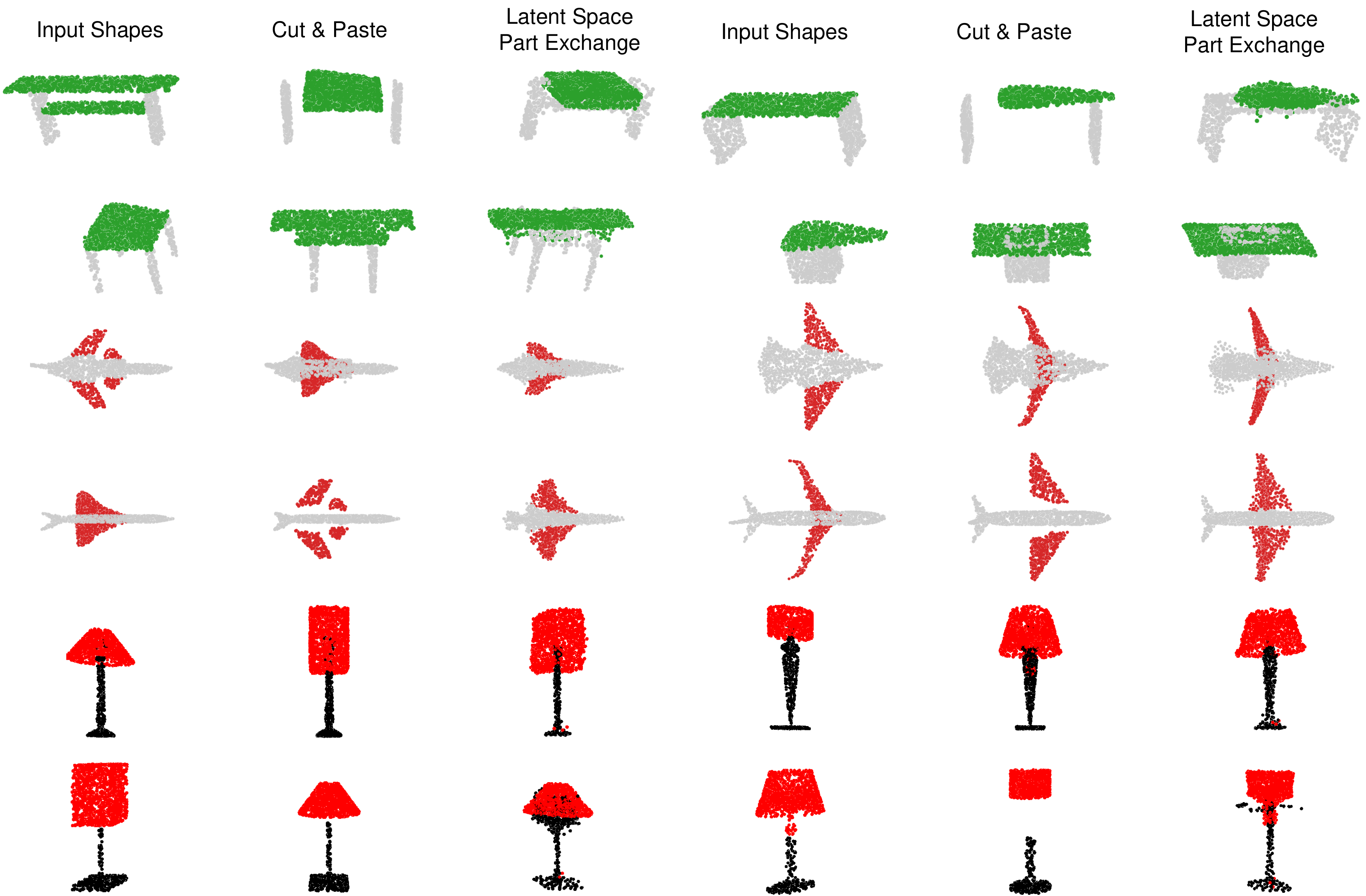}{Part replacement visualization and comparison. By operating in the latent space, more natural replacement results could be obtained, without suffering from the detachment problems as with simple Cut \& Paste method.}{fig:supp-partrep}{hbtp}

\end{document}